\documentclass[letterpaper]{article} 
\usepackage[preprint]{aaai2027}  
\usepackage[hyphens]{url}  
\usepackage{graphicx} 
\def\UrlFont{\rm}  
\usepackage{natbib}  
\usepackage{caption} 
\usepackage{amsmath}
\usepackage{amssymb}
\usepackage{amsthm}
\usepackage{booktabs}

\newtheorem{proposition}{Proposition}

\title{Time--Frequency Geometric Cross-Attention\\for Chunked Vision--Language--Action Models}

\author{
    Shengye Dong, Haochen Niu, Hao Liu, Peiwen Lin, Chuang Wang, Shanmin Pang
}
\affiliations{
}

\begin{document}

\maketitle

\begin{abstract}
Modern vision--language--action (VLA) policies predict a whole \emph{chunk} of actions: one to two seconds of coordinated motion emitted in a single forward pass. Yet an action chunk is essentially a short multivariate trajectory, and inside these models it is represented as a sequence of generic per-timestep hidden tokens and decoded by a linear head. This representation tends to under-serve two structures of motion. The first is \textbf{frequency}: a chunk superimposes a smooth global trend and fine corrective motion across several time scales, and a single token entangles them. The second is \textbf{cross-phase geometric relationships}: within a chunk, the motions of different phases (reach, contact, grasp adjustment, settling) unfold along very different directions---near-orthogonal in the representation space, yet tightly related for the task, and arising across the \emph{time axis} (between phases) rather than at the same instant. Dot-product attention scores alignment by an inner product, so it favors similar (aligned) tokens and is least sensitive exactly near orthogonality, leaving such cross-phase relationships to be recovered by the network through a detour.

We introduce \textbf{Time--Frequency Geometric Cross-Attention (TFGCA)}, a drop-in module that repairs both blind spots. TFGCA uses a \emph{per-dimension learnable stationary wavelet transform} (SWT) to decompose the action chunk into time--frequency tokens, and each time token then retrieves information from them through a cross-attention that fuses the dot product (similarity) with the \emph{wedge}-product magnitude (sensitive to near-orthogonality) via a learnable weight. The module uses a zero-initialized residual, so it reproduces the base behavior exactly at initialization and can be dropped onto a pretrained VLA and fine-tuned jointly with it. Relative to the same-source base, TFGCA improves the near-saturated in-distribution LIBERO by $+1.5$ on average, the OOD benchmark LIBERO-Plus by $+6.3$, and the randomized average under RoboTwin domain randomization by $+28.5$, and raises the overall success rate on three real-robot AgiBot A2 tasks by \textbf{$+11.67$ percentage points}---with the gains larger under out-of-distribution conditions.
\end{abstract}

\section{Introduction}
\label{sec:intro}

Vision--language--action models have converged on a common output interface: rather than emitting one action per step, they predict a \emph{chunk} of $T$ future actions at once and execute some prefix of it before re-planning. Flow-matching and diffusion policies such as $\pi_0$~\citep{black2024pi0}, $\pi_{0.5}$~\citep{pi05_2025}, RDT-1B~\citep{rdt2024}, and the action-chunking transformer ACT~\citep{act2023} all commit to a $T\times D$ trajectory in a single forward pass, where $T$ is the chunk horizon and $D$ the number of action dimensions. Chunked prediction improves temporal consistency and lets a policy plan short-horizon coordination instead of reacting one step at a time.

The chunk is therefore a \emph{structured object}: a short multivariate trajectory whose value lies in how its dimensions move together over time. Current VLAs, however, treat it as an unstructured one. The action tokens leaving the transformer are generic hidden vectors, and a single linear head maps each to an action. Two properties of motion that a trajectory model should exploit are left implicit.

\paragraph{Frequency structure.} A manipulation chunk overlays motion at several time scales. A slow transport component carries the end-effector across the workspace, while faster components perform contact-time corrections, grasp adjustments, and settling. Packed into one hidden token per timestep, these scales are entangled, and the model has no explicit handle on ``the trend'' versus ``the correction.''

\paragraph{Cross-phase geometric relationships.} Skilled manipulation unfolds over time into several phases---reach, contact, grasp adjustment, settling---that are related for the task yet can act along very different directions in the representation space, near-orthogonal in the limit. Where does this near-orthogonal ``division of labor'' occur? Our analysis of RoboTwin 2.0 bimanual data shows that it hardly appears at the same instant, but unfolds sequentially along the \emph{time axis}---trajectories successively occupy a set of near-orthogonal subspaces in ordered temporal phases, which we call \textbf{Temporal Orthogonal Division-of-Labor (TO-DoL)} (metrics and robustness in Appendix~H). Dot-product attention, however, measures only \emph{alignment} and is least sensitive exactly near orthogonality, so it must spend extra capacity to re-encode such cross-phase relationships as similarity. This motivates an inductive hypothesis: giving attention an extra scoring channel sensitive to near-orthogonal components (the wedge product) may let each time token retrieve such \emph{cross-phase} relationships more directly from the whole chunk's time--frequency evidence (Section~\ref{sec:method}).

Frequency structure and geometric attention have precedents in time-series forecasting (e.g., compact forecasters that pair a wavelet transform with a geometric attention). The VLA setting is different in kind: an action chunk is a short, non-periodic trajectory over \emph{actual control degrees of freedom}, and any added module rides on a pretrained multi-billion-parameter backbone whose learned weights it must not disturb at initialization. We therefore design a time--frequency geometric attention module, TFGCA, purpose-built for VLA action prediction; its design parts and contributions follow.

\paragraph{Contributions.} We introduce \textbf{TFGCA}, a time--frequency geometric cross-attention module for VLA action prediction that supplies chunked policies with two motion structures they previously ignored: multi-scale frequency content, and cross-phase geometric relationships within the chunk that fall along near-orthogonal directions. Its two core parts are as follows (full method in Section~\ref{sec:method}).
\begin{itemize}
\item \textbf{(C1, frequency) Per-control-dimension action-space wavelet tokenization.} We project the hidden action tokens to action space and apply a \emph{per-dimension learnable} stationary wavelet transform (SWT), turning each control dimension's trajectory into length-preserving, multi-resolution frequency tokens---supplying the multi-scale frequency structure a base VLA leaves implicit.
\item \textbf{(C2, geometry) A time-to-time--frequency geometric cross-attention.} Each time token retrieves information from the whole chunk's time--frequency tokens through a \emph{learnable} blend of the dot product (similarity) and the wedge product. Proposition~\ref{prop:decomp} shows the wedge is sensitive to near-orthogonal directions in the representation space and can realize orderings the dot channel alone cannot; from this we propose a \emph{geometry-sensitive inductive hypothesis}---that cross-phase, cross-time-scale correlations within a chunk often fall along dot-insensitive near-orthogonal directions, for which the wedge channel can provide a direct score---whose real-task payoff is answered empirically in Section~\ref{sec:exp}.
\end{itemize}

\begin{figure*}[t]
\centering
\includegraphics[width=0.92\textwidth]{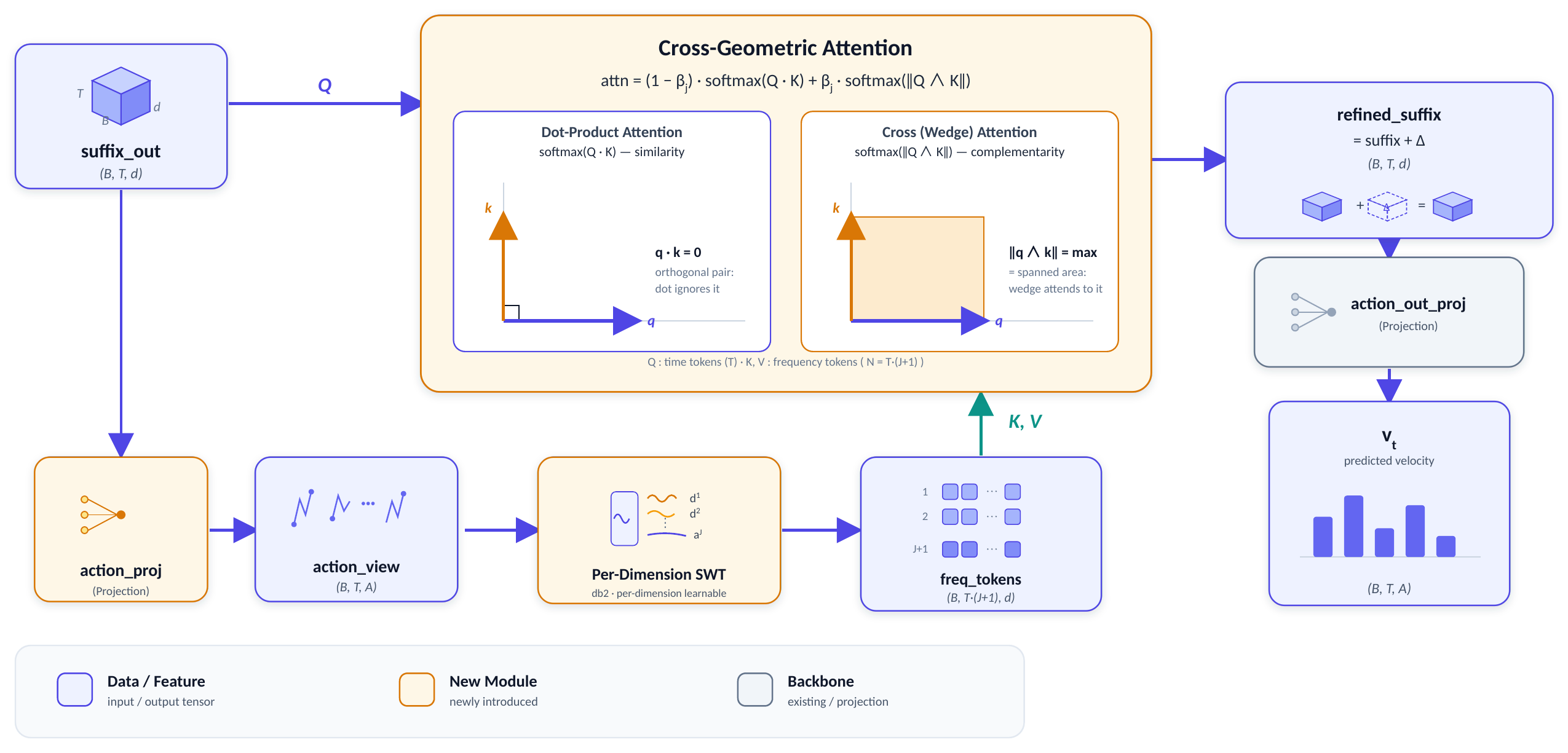}
\caption{TFGCA architecture: a drop-in module between the VLA transformer output and the action head. The frequency branch (action-space projection $\to$ per-dimension SWT $\to$ frequency tokens) supplies the keys/values for the geometric cross-attention, whose queries are the time tokens; its output is written back residually and passed through the action head to give the velocity $v_\theta$. (Blue: data tensors; amber: new modules; gray: backbone.)}
\label{fig:arch}
\end{figure*}

\section{Related Work}

\paragraph{Vision--language--action models and action chunking.} VLAs map instructions and observations to robot actions by adapting large pretrained vision--language backbones. Discrete-token approaches such as RT-2~\citep{rt2_2023} and OpenVLA~\citep{openvla2024} autoregress action tokens, while Octo~\citep{octo2024}, diffusion policy~\citep{diffusionpolicy2023}, $\pi_0$~\citep{black2024pi0}, $\pi_{0.5}$~\citep{pi05_2025}, and RDT-1B~\citep{rdt2024} predict continuous action chunks with diffusion or flow matching~\citep{flowmatching2023}. Action chunking, introduced for imitation with ACT~\citep{act2023}, predicts a block of future actions and executes a prefix, improving temporal coherence, often with temporal ensembling across overlapping chunks to reduce discontinuities. TFGCA operates inside this chunked-prediction interface: it is a representation of the chunk placed between the backbone and the action head, orthogonal to how the chunk is supervised (flow matching, diffusion, or regression).

\paragraph{Frequency-domain methods for sequences and actions.} Frequency-domain inductive biases are well studied in forecasting: Autoformer~\citep{autoformer2021} and FEDformer~\citep{fedformer2022} inject seasonal--trend decomposition and Fourier attention, FreTS~\citep{frets2023} learns in the frequency domain, and SimpleTM~\citep{simpletm2025} couples a stationary wavelet transform with a geometric attention. Within VLAs, FAST~\citep{fast2025} applies a DCT to action chunks to produce compact discrete tokens for autoregressive policies, using frequency for \emph{tokenization and compression}. Our use is different: we treat frequency as a multi-resolution, time-localized \emph{feature} that attention consults, on a continuous flow-matching policy; the stationary (non-decimating) wavelet preserves sequence length, so every band stays aligned with the original timesteps and can serve directly as attention keys.

\paragraph{Geometric structure in attention.} Geometric-algebra networks (the Geometric Algebra Transformer~\citep{gatr2023}; Clifford neural layers~\citep{clifford2023}) represent data as multivectors and act on them equivariantly. We borrow only a single scalar---the wedge-product magnitude---and use it as an attention score alongside the inner product, keeping the module lightweight and free of multivector bookkeeping while capturing the one property we need: sensitivity to orthogonality. Unlike iTransformer~\citep{itransformer2024}, which treats variables (dimensions) as attention tokens, our frequency tokens jointly encode the multi-dimensional action at each time--band position, and the cross-attention retrieves along time--frequency rather than attending between dimensions.

\section{Preliminaries}

\paragraph{VLA action prediction as trajectory forecasting.} We build on a flow-matching VLA ($\pi_{0.5}$). Given observations (images and a language instruction) and proprioceptive state, the policy predicts an action chunk $a \in \mathbb{R}^{T \times D}$, where $T$ is the chunk horizon and $D$ the number of action dimensions. Flow matching trains a velocity field: sample noise $\epsilon$ and a time $\tau \in [0,1]$, form the interpolant $x_\tau = \tau\,\epsilon + (1-\tau)\,a$, and regress the network output $v_\theta(x_\tau, \cdot)$ toward the target velocity $u = \epsilon - a$ with the loss
\begin{equation}
\mathcal{L}_{\mathrm{FM}} = \mathbb{E}_{\tau,\epsilon}\,\big\lVert v_\theta(x_\tau,\cdot) - (\epsilon - a) \big\rVert^2 .
\end{equation}
The backbone is a decoder transformer that jointly attends over a multimodal prefix (image and language tokens) and an action suffix; the last $T$ suffix positions produce hidden action tokens $H \in \mathbb{R}^{T \times d}$ (with $d$ the model width), which a linear action head maps to $v_\theta$. TFGCA is a transformation $H \mapsto \tilde H$ inserted just before this head, applied identically in training and in each denoising step at inference.

\paragraph{Dot-product attention prefers similarity.} For queries $Q$ and keys $K$ split into heads of dimension $d_h$, attention weights come from $q_i^\top k_j / \sqrt{d_h}$, largest when $q_i$ and $k_j$ are aligned and vanishing when orthogonal, so \emph{in a given representation} two orthogonal tokens receive only a small direct score. A deep network can of course relearn features that re-encode orthogonal relationships as aligned and route them through the similarity channel (which is why existing dot-product policies work); but that spends representational capacity, and an explicit orthogonality-sensitive signal removes the detour (formalized in the scope remark of Proposition~\ref{prop:decomp}).

\paragraph{The wedge product measures orthogonality.} The wedge-product magnitude is the area of the parallelogram $q,k$ span, with a closed form via the Cauchy--Schwarz identity (no antisymmetric-tensor construction):
\begin{equation}
\lVert q \wedge k \rVert^2 = \lVert q \rVert^2\,\lVert k \rVert^2 - (q^\top k)^2 .
\end{equation}
It is geometrically complementary to the inner product: zero when $q \parallel k$, maximal (for fixed norms) when $q \perp k$. Using both scores gives attention access to \emph{alignment} and \emph{orthogonality} at once.

\section{Method}
\label{sec:method}

TFGCA transforms the hidden action tokens $H \in \mathbb{R}^{T\times d}$ into refined tokens $\tilde H \in \mathbb{R}^{T\times d}$ in three stages (overall structure in Figure~\ref{fig:arch}): a projection to a per-dimension control space (Section~\ref{subsec:actionproj}), a per-dimension wavelet tokenization (Section~\ref{subsec:swt}), and a geometric cross-attention that fuses the two (Section~\ref{subsec:attn}). Section~\ref{subsec:init} covers the identity-at-initialization construction that makes the module safe to attach to a pretrained policy, and Section~\ref{subsec:theory} gives the wedge's theoretical properties and the mechanism explanation.

\subsection{Action-Space Projection}
\label{subsec:actionproj}

The transformer's action tokens live in an abstract $d$-dimensional space with no explicit per-control-dimension meaning, whereas the structure we want to model---the frequency content of each action dimension and their coordination---is defined in the action coordinates the policy actually outputs. We therefore first project each token to a \emph{per-dimension control-space view},
\begin{equation}
Z^{\mathrm{act}} = \mathrm{ActionProj}(H) \in \mathbb{R}^{T \times D},
\end{equation}
with a dedicated linear map. The physical meaning of the $D$ output channels depends on the embodiment's action parameterization: in RoboTwin 2.0 and AgiBot A2 they correspond to joint-space actions; in LIBERO, the $7$ channels correspond to Cartesian end-effector increments (3-D translation, 3-D rotation) and a 1-D binary gripper command, not $7$ joint angles. The projection is initialized so that $Z^{\mathrm{act}}$ is nonzero from the first step, which matters for gradient flow (Section~\ref{subsec:init}). To give this projection explicit physical grounding, we add an auxiliary alignment loss that supervises $Z^{\mathrm{act}}$ toward the \textbf{stop-gradient} ground-truth action chunk $a$,
\begin{equation}
\mathcal{L}_{\mathrm{align}} = \big\lVert Z^{\mathrm{act}} - \mathrm{sg}(a) \big\rVert^2 ,
\end{equation}
where $\mathrm{sg}(\cdot)$ is the stop-gradient (only the projection is updated, no gradient flows to the target); training minimizes $\mathcal{L}=\mathcal{L}_{\mathrm{FM}}+\lambda\,\mathcal{L}_{\mathrm{align}}$ for a small weight $\lambda$ (per-benchmark values in Appendix~B). This loss is part of the full model; removing it is the ``w/o~alignment'' ablation (Section~\ref{subsec:libero}).

\subsection{Per-Dimension Learnable SWT Tokenization}
\label{subsec:swt}

Given the control-space view $Z^{\mathrm{act}}$, we decompose each action dimension's length-$T$ sequence into multiple frequency bands with a stationary wavelet transform (SWT; Figure~\ref{fig:swt}). Unlike the standard (decimated) wavelet transform, the SWT does not downsample: every band has length $T$, so bands stay aligned with the original timesteps and can serve as attention keys. One level splits an input approximation $a^{(j)}$ into a coarser approximation and a detail band using a low-pass/high-pass filter pair $(h,g)$ applied at dilation $2^{j}$,
\begin{equation}
a^{(j+1)}_t = \sum_{k} h_k\, a^{(j)}_{t + 2^{j} k}, \quad
d^{(j+1)}_t = \sum_{k} g_k\, a^{(j)}_{t + 2^{j} k},
\end{equation}
recursing on the approximation branch. With $J$ levels this yields $J$ detail bands and one final approximation, $\{d^{(1)},\dots,d^{(J)}, a^{(J)}\}$, ordered high to low frequency. We make three design choices specific to robot actions.

\paragraph{Per-dimension learnable filters.} Each action dimension gets its own filter pair, initialized to Daubechies-2 (\texttt{db2}, length 4; Haar is an option). \texttt{db2}'s two vanishing moments make its detail coefficients cancel locally linear motion and respond only to curvature, suiting the long, smooth segments common in manipulation; filters are learnable, so each dimension specializes a task-relevant decomposition on top of this strong prior.

\paragraph{Per-dimension DC removal.} A smooth demonstration's energy sits mostly at zero frequency (a near-constant offset) and would swamp the detail bands, so we subtract each dimension's mean over time before decomposition. We deliberately do \emph{not} divide by the standard deviation: the amplitude contrast between high-motion moments (contact, grasp) and stationary segments is exactly the signal cross-attention should key on, and normalizing it away would amplify noise in still segments.

\paragraph{Optional derivative pre-processing.} Because action sequences can be low-frequency-dominated even after DC removal, the view may be finite-differenced before the SWT to move task-relevant energy into higher bands (the order used per benchmark is in Appendix~B).

At each (time, band) position, a per-band linear map embeds the whole \emph{$D$-dimensional action vector} at that position into the model width $d$ ($\mathbb{R}^{D}\!\to\!\mathbb{R}^{d}$), plus a band-type embedding and a time-position embedding. This step \emph{jointly encodes} the multiple action dimensions at the same time--band position into a single token---cross-dimension structure is carried as token content rather than paired dimension-by-dimension by the subsequent attention. Stacking the $J{+}1$ bands over $T$ timesteps and flattening (band index fastest) gives a sequence of frequency tokens
\begin{equation}
F \in \mathbb{R}^{\,T(J+1)\times d}.
\end{equation}

\begin{figure}[t]
\centering
\includegraphics[width=\linewidth]{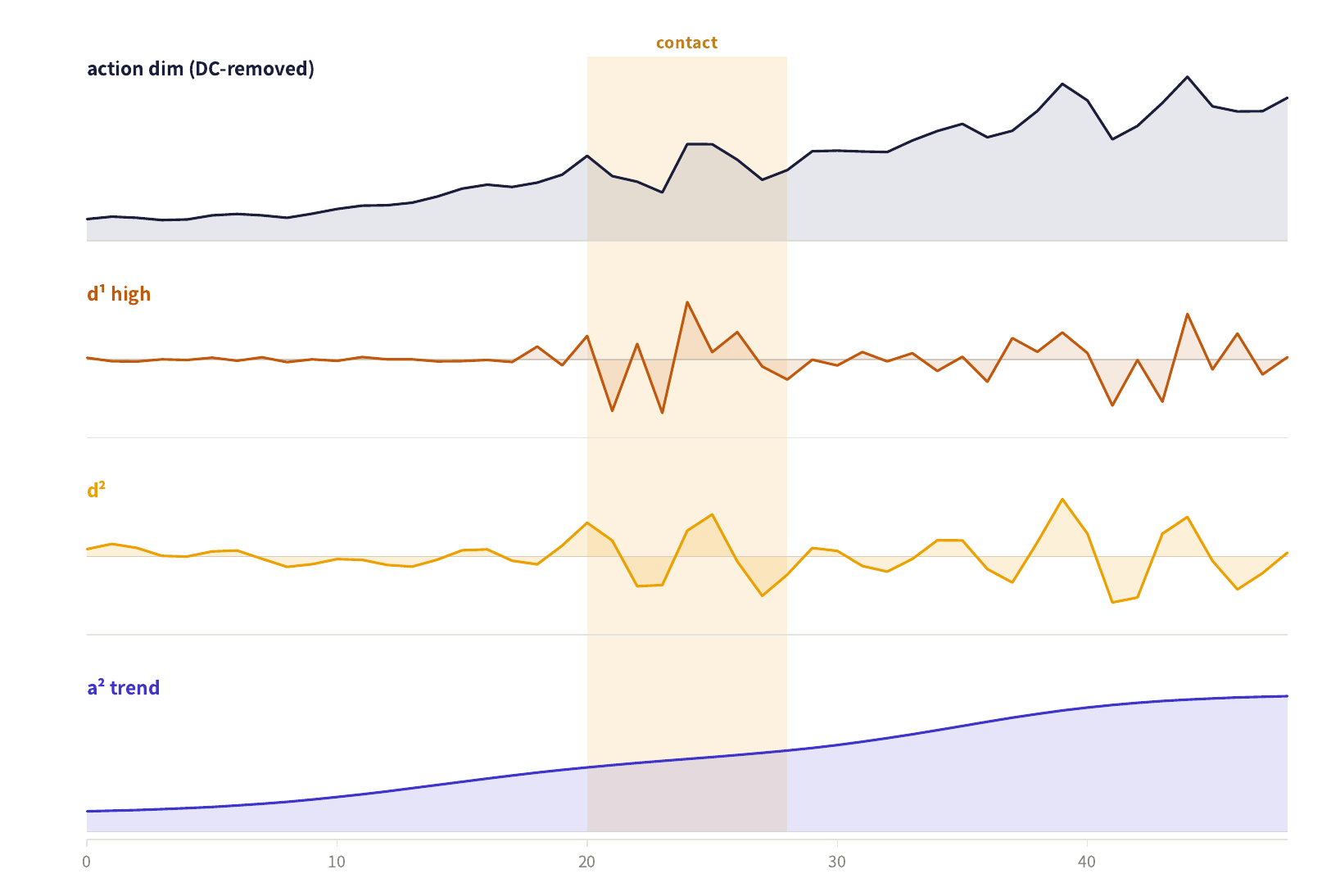}
\caption{A per-dimension SWT (\texttt{db2}) decomposes a single action-dimension sequence (top, DC-removed per dimension) into multi-resolution bands. High-frequency jitter at the contact moment (shaded) concentrates in $d^{(1)}$, while the final approximation $a^{(2)}$ retains the smooth trend---trend and correction are decoupled into different bands.}
\label{fig:swt}
\end{figure}

\subsection{Dot-Plus-Wedge Cross-Attention}
\label{subsec:attn}

TFGCA's attention is a \emph{cross}-attention: queries are the time tokens $H$, and keys/values are the frequency tokens $F$. Each time token thus retrieves evidence from the whole chunk's time--frequency tokens (each token jointly encoding multi-dimensional action content at one time--band position). With multi-head projections $Q=W_Q H$, $K=W_K F$, $V=W_V F$ (head dimension $d_h$, scale $s = 1/\sqrt{d_h}$), we form two score matrices,

\begin{align}
S^{\mathrm{dot}}_{ij} &= s\,(q_i^\top k_j), \\
S^{\mathrm{wed}}_{ij} &= s\,\sqrt{\lVert q_i\rVert^2\lVert k_j\rVert^2 - (q_i^\top k_j)^2 } .
\end{align}

We convert each score matrix to attention weights with its own softmax and blend them with a learnable scalar. Let $\beta = \sigma(\ell)$ be a global learnable mixing weight (a single scalar logit $\ell$, initialized so $\beta = 0.5$). Then
\begin{equation}
A = (1-\beta)\,\mathrm{softmax}(S^{\mathrm{dot}}) \;+\; \beta\,\mathrm{softmax}(S^{\mathrm{wed}}).
\end{equation}
This two-softmax-then-blend form (rather than one softmax over a pre-summed score) keeps each channel a proper distribution and lets the model weight \emph{alignment} (dot) against \emph{orthogonality} (wedge) with the learnable $\beta$. The attended value is projected and added residually,
\begin{equation}
\tilde H = H + W_O\big(A\,V\big).
\end{equation}

\subsection{Identity at Initialization and Integration}
\label{subsec:init}

The output projection $W_O$ is zero-initialized, so at initialization $\tilde H = H$ and the module reproduces the base prediction exactly on attachment. This is the standard zero-initialized residual paradigm~\citep{controlnet2023,flamingo2022,lora2022,rezero2021}. So that the module can leave the identity map, we zero only $W_O$ and keep the value path (band embeddings and action-space projection) nonzero: the initial gradient $\nabla_{W_O}\mathcal{L}=\delta\,(AV)^\top$ is then generically nonzero, whereas zeroing $V$ too would trap the module at the identity (formal statement in Appendix~I). This preserves function \emph{at initialization only}---it eases drop-in but does \emph{not} guarantee the fine-tuned model beats the base (Section~\ref{subsec:robotwin} and Section~\ref{sec:discussion}). The added parameters are negligible (about $4.23$M, roughly $0.10\%$ of the $\pi_{0.5}$ backbone; breakdown in Appendix~G.2).

\subsection{Theoretical Properties and Mechanism}
\label{subsec:theory}

We record the property that formalizes the wedge channel; its proof is short and follows from the definitions.

\begin{proposition}[Alignment--orthogonality decomposition and order separation]
\label{prop:decomp}
For any nonzero $q,k\in\mathbb{R}^{d_h}$, $(q^\top k)^2+\lVert q\wedge k\rVert^2=\lVert q\rVert^2\lVert k\rVert^2$; after norm normalization, $\hat{s}=q^\top k/(\lVert q\rVert\lVert k\rVert)$ and $\hat{w}=\lVert q\wedge k\rVert/(\lVert q\rVert\lVert k\rVert)$ satisfy $\hat{s}^2+\hat{w}^2=1$ (with $\hat{s}=\cos\theta$ and $\hat{w}=\lvert\sin\theta\rvert$ for $\theta$ the angle between $q$ and $k$). Thus the dot and wedge scores are two \emph{orthogonal coordinates} on the unit circle of pairwise relatedness---the dot product reads only $\hat{s}$ and is blind to the $\hat{w}$ axis along which near-orthogonal relationships vary. Consequently there is an \emph{order separation}: fix $q\neq0$ and two equal-norm keys $\lVert k_a\rVert=\lVert k_b\rVert=\kappa$ with $q^\top k_a > q^\top k_b \ge 0$; then
\begin{itemize}
\item the \textbf{dot channel} under any monotone softmax has $A^{\mathrm{dot}}(k_a)>A^{\mathrm{dot}}(k_b)$---it must place more weight on the more-aligned $k_a$ and can never favor the more-orthogonal $k_b$;
\item the \textbf{wedge channel} has $\lVert q\wedge k_a\rVert<\lVert q\wedge k_b\rVert$ and hence $A^{\mathrm{wed}}(k_b)>A^{\mathrm{wed}}(k_a)$.
\end{itemize}
Therefore, once $\beta$ exceeds a threshold $\beta^\star=\Delta^{\mathrm{dot}}/(\Delta^{\mathrm{dot}}+\Delta^{\mathrm{wed}})\in(0,1)$ (with $\Delta^{\mathrm{dot}}=A^{\mathrm{dot}}(k_a)-A^{\mathrm{dot}}(k_b)>0$, $\Delta^{\mathrm{wed}}=A^{\mathrm{wed}}(k_b)-A^{\mathrm{wed}}(k_a)>0$), the blend $A=(1-\beta)A^{\mathrm{dot}}+\beta A^{\mathrm{wed}}$ favors the near-orthogonal key ($A(k_b)>A(k_a)$)---a ranking \emph{no dot-product score can realize}. Such a $\beta$ exists and is reachable by the learnable weight; not every $\beta>0$ flips the order.
\end{proposition}

\begin{proof}
The first identity is the Lagrange (Cauchy--Schwarz) identity; dividing by $\lVert q\rVert^2\lVert k\rVert^2$ gives $\hat{s}^2+\hat{w}^2=1$. The two per-channel orderings follow from monotonicity of the softmax and, for the wedge channel, from $\lVert q\wedge k\rVert^2=\lVert q\rVert^2\kappa^2-(q^\top k)^2$ with $(q^\top k_a)^2>(q^\top k_b)^2$. For the blend, $A(k_b)-A(k_a)=\beta\,\Delta^{\mathrm{wed}}-(1-\beta)\,\Delta^{\mathrm{dot}}$, which is positive iff $\beta>\Delta^{\mathrm{dot}}/(\Delta^{\mathrm{dot}}+\Delta^{\mathrm{wed}})=\beta^\star$; since $\Delta^{\mathrm{dot}},\Delta^{\mathrm{wed}}>0$ we have $\beta^\star\in(0,1)$, and at $\beta=1$ the blend equals the wedge channel, which already flips the order.
\end{proof}

\paragraph{Representational-complexity view.} $\lVert q\wedge k\rVert^2$ is a quadratic form in $(q,k)$; a dot product can reproduce it only after a quadratic feature lift $\phi(x)=\mathrm{vec}(xx^\top)$ of dimension up to $O(d_h^2)$, whereas the wedge computes it in closed form at $O(d_h)$ via Cauchy--Schwarz.

\paragraph{Scope of the claim.} This separation holds at a \emph{fixed representation}. Because $Q,K$ come from learnable projections, a deep network can in principle relearn features that re-encode orthogonal relationships as aligned ones and route them through the dot channel, which is why existing dot-product policies work. Proposition~\ref{prop:decomp} is therefore not an impossibility result for dot-product models; it states that the wedge supplies this ordering \emph{directly, without relearning features or a quadratic lift}, a low-cost inductive bias whose payoff on real tasks is answered empirically in Section~\ref{sec:exp} (Figure~\ref{fig:toy} gives a $d_h{=}2$ instance).

\paragraph{Mechanism.} Figure~\ref{fig:toy} visualizes Proposition~\ref{prop:decomp} on a three-phase toy example: after the softmax the dot and wedge channels put their weight on opposite (aligned vs.\ orthogonal) keys, and the learnable $\beta$ interpolates between them. From this we \emph{hypothesize} that cross-phase correlations within a chunk often fall along near-orthogonal directions the dot product misses and that the wedge can score them directly---supported by the TO-DoL analysis of Section~\ref{sec:intro}. Its real benefit is evaluated empirically (Section~\ref{sec:exp}); the causal attribution ($Q/K$ routing analysis and OOD ablations) is left to future work (Section~\ref{sec:discussion}).

\begin{figure}[t]
\centering
\includegraphics[width=\linewidth]{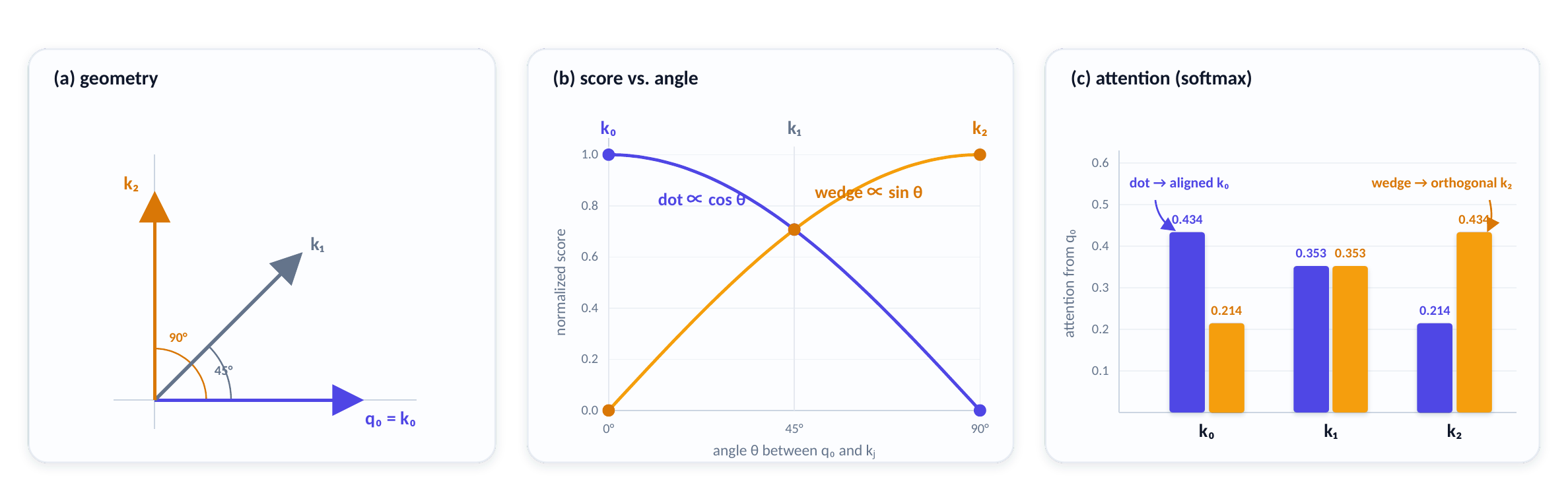}
\caption{Toy example ($d_h{=}2$, $q_0$ along axis A). \textbf{(a)}~the three keys make angles $0^\circ/45^\circ/90^\circ$ with $q_0$; \textbf{(b)}~normalized scores versus angle---$\mathrm{dot}\propto\cos\theta$ (falls with angle), $\mathrm{wedge}\propto\sin\theta$ (rises with angle), crossing at $45^\circ$, exactly the $\hat{s}^2+\hat{w}^2=1$ of Proposition~\ref{prop:decomp}; \textbf{(c)}~after the softmax the two attention distributions are \emph{flipped}: the dot channel gives its weight to the aligned $k_0$ (0.434), the wedge channel to the orthogonal $k_2$ (0.434).}
\label{fig:toy}
\end{figure}

\section{Experiments}
\label{sec:exp}

We evaluate TFGCA in four settings: LIBERO~\citep{libero2023} (Section~\ref{subsec:libero}, full-suite in-distribution performance and component ablations), LIBERO-Plus~\citep{liberoplus2025} (Section~\ref{subsec:liberoplus}, OOD generalization), RoboTwin 2.0~\citep{robotwin2025} (Section~\ref{subsec:robotwin}, single-task, clean and domain-randomized), and the AgiBot A2 real robot (Section~\ref{subsec:agibot}). Dataset sources, per-benchmark training and evaluation hyperparameters (optimizer, learning rate, chunk length, SWT levels, alignment-loss weight, etc.), and the training hardware and wall-clock times are all listed in Appendices~A, B, and~G; on each benchmark TFGCA and its same-source reproduced $\pi_{0.5}$ base share consistent data, training budget, and evaluation, differing only in the TFGCA-specific SWT/geometry configuration.

The empirical emphasis across these four settings is TFGCA's \textbf{robustness under out-of-distribution (OOD) conditions}: beyond the near-saturated in-distribution LIBERO, evaluation covers the unseen perturbations of LIBERO-Plus, RoboTwin domain randomization, and the real AgiBot A2 robot. We report each setting separately.

\subsection{LIBERO: Full-Suite Experiments}
\label{subsec:libero}

\paragraph{Setup.} We evaluate TFGCA attached to $\pi_{0.5}$ on the four LIBERO task suites (Spatial, Object, Goal, Long/L10), against the $\pi_{0.5}$ base (reported value and our reproduction); a broader comparison with external methods is deferred to Appendix~C. Because the base success rate is already high ($\sim$97\% average), LIBERO is a near-saturated in-distribution setting. LIBERO is also single-arm (end-effector-delta control), so the specific bimanual time-axis division-of-labor of TO-DoL (Section~\ref{sec:intro}) does not arise here; even so, the wedge remains a general scoring channel sensitive to the near-orthogonal components the dot product misses, and can still supply complementary information in-distribution. The full model is weakly but consistently best on the four-suite average below, though at LIBERO's saturation these gaps are small and we do not read them as isolating the wedge's effect. Both the ablations and the full TFGCA are trained on our reproduced $\pi_{0.5}$. Table~\ref{tab:libero} reports the full TFGCA together with per-component ablations: ``w/o~SWT'' removes the wavelet frequency branch (C1); ``w/o~geometry'' removes the wedge channel, reducing the attention to pure dot product (C2); ``w/o~alignment'' disables the optional alignment loss of Section~\ref{subsec:actionproj}.

\begin{table}[t]
\centering
\small
\setlength{\tabcolsep}{4pt}
\begin{tabular}{lccccc}
\toprule
Method & Spatial & Object & Goal & Long & Avg \\
\midrule
$\pi_{0.5}$ (reported) & 98.8 & 98.2 & 98.0 & 92.4 & 96.9 \\
$\pi_{0.5}$ (repro.) & 95.2 & 99.6 & 97.2 & 94.6 & 96.7 \\
\textbf{TFGCA (full)} & \textbf{98.5} & \textbf{99.4} & \textbf{97.7} & \textbf{97.0} & \textbf{98.2} \\
\quad w/o SWT & 96.8 & 99.1 & 97.1 & 97.8 & 97.7 \\
\quad w/o geometry & 97.2 & 99.7 & 97.6 & 96.5 & 97.8 \\
\quad w/o alignment & 94.9 & 99.5 & 98.5 & 95.5 & 97.1 \\
\bottomrule
\end{tabular}
\caption{LIBERO success rate (SR\%) on the four suites and the average. TFGCA and the reproduced $\pi_{0.5}$ base are run under the same protocol; the $\pi_{0.5}$ reference is the reported value~\citep{pi05_2025}. A broader comparison against 11 external methods (DP, OpenVLA, SpatialVLA, CoT-VLA, $\pi_0$/-FAST, GR00T, OpenVLA-OFT, Fast-WAM, X-VLA) is given in Appendix~C.}
\label{tab:libero}
\end{table}

\paragraph{On near-saturated LIBERO, TFGCA still delivers a consistent average gain.} Relative to our reproduced $\pi_{0.5}$ base (96.7\% average), full TFGCA lifts the four-suite average to \textbf{98.2\%}---a $+1.5$ net gain despite only ${\sim}3$ points of headroom, and it raises the hardest suite (Long) from 94.6 to \textbf{97.0}; the primary evidence for OOD robustness comes from the later evaluations. The component ablation illustrates the joint necessity of the three components in-distribution: removing the wavelet frequency branch (w/o~SWT), the wedge channel (w/o~geometry), or the alignment loss (w/o~alignment) each lowers the \emph{four-suite average} (to 97.7, 97.8, and 97.1 respectively), with w/o~alignment dropping the most.

\subsection{LIBERO-Plus: OOD Generalization}
\label{subsec:liberoplus}

\paragraph{Setup.} LIBERO-Plus applies seven families of out-of-distribution perturbations to LIBERO tasks---camera viewpoint (Camera), robot initialization (Robot), language rephrasing (Language), lighting (Light), background (Background), sensor noise (Noise), and object layout (Layout)---reporting a success rate per family and a Total. All policies are trained only on clean (unperturbed) LIBERO and then tested under each perturbation, so LIBERO-Plus measures OOD generalization to unseen perturbations. We compare TFGCA attached to $\pi_{0.5}$ against the same-source reproduced base, both self-tested under the same protocol (Table~\ref{tab:liberoplus}).

\begin{table*}[t]
\centering
\small
\setlength{\tabcolsep}{5pt}
\begin{tabular}{lcccccccc}
\toprule
Method & Camera & Robot & Language & Light & Background & Noise & Layout & Total \\
\midrule
$\pi_{0.5}$ (our reproduced base) & 42.1 & 62.9 & 77.2 & 96.3 & 88.4 & 38.4 & 79.6 & 66.7 \\
\textbf{TFGCA (ours, 30k)} & 51.5 & 73.4 & 81.8 & 96.4 & 87.8 & 52.3 & 81.7 & \textbf{73.0} \\
\quad w/o SWT & 46.7 & 74.2 & 80.7 & 93.8 & 87.4 & 54.2 & 82.4 & 72.1 \\
\quad w/o geometry & 41.1 & 67.9 & 81.0 & 94.6 & 87.1 & 49.8 & 83.4 & 69.7 \\
\quad w/o alignment & 49.6 & 70.5 & 79.1 & 97.4 & 89.1 & 51.1 & 82.7 & 71.9 \\
\bottomrule
\end{tabular}
\caption{LIBERO-Plus success rate (\%) across seven perturbation families and the Total. All rows are self-tested under the same protocol: our reproduced same-source base, the TFGCA-augmented model, and its three component ablations (w/o SWT, w/o geometry, w/o alignment; cf.\ Table~\ref{tab:libero}).}
\label{tab:liberoplus}
\end{table*}

\paragraph{On LIBERO-Plus, TFGCA improves broadly over its same-source base.} Under the same protocol, TFGCA beats the reproduced base (66.7) on six of the seven perturbations, and the Total rises from 66.7 to 73.0 (\textbf{$+6.3$}), with the gains concentrated on the base's hardest families---Noise, Robot, and Camera (per-family in Table~\ref{tab:liberoplus}; the full per-suite $\times$ per-perturbation results are in Appendix~E). The component ablation mirrors LIBERO but with a different ordering under OOD: removing the wedge channel (w/o~geometry) drops the Total the most (to 69.7, $-3.3$), ahead of w/o~alignment (71.9) and w/o~SWT (72.1)---the reverse of the near-saturated in-distribution case (Section~\ref{subsec:libero}), where alignment mattered most, consistent with the geometry channel contributing more precisely under out-of-distribution perturbations.

\subsection{RoboTwin: Single-Task (Clean and Randomized)}
\label{subsec:robotwin}

\paragraph{Setup.} All policies are trained only on RoboTwin 2.0 \textbf{clean (no domain randomization)} demonstrations for \textbf{30k steps}, then tested separately in \textbf{clean} and \textbf{randomized} environments. Since randomization is never seen during training, the randomized column measures \textbf{zero-shot generalization} to unseen perturbations. Table~\ref{tab:robotwin} compares TFGCA (attached to $\pi_{0.5}$) against the $\pi_{0.5}$ base under the same protocol; a full comparison with RoboTwin leaderboard baselines (DP, ACT, DP3, RDT, $\pi_0$) is in Appendix~D.

\begin{table}[t]
\centering
\small
\setlength{\tabcolsep}{5pt}
\begin{tabular}{lcc}
\toprule
Task & $\pi_{0.5}$ & \textbf{TFGCA (ours)} \\
\midrule
open\_microwave     & 93 / 27  & 88 / 36 \\
stamp\_seal         & 15 / 5   & 19 / 3 \\
handover\_block     & 57 / 13  & 72 / 11 \\
turn\_switch        & 41 / 28  & 54 / 61 \\
stack\_bowls\_three & 83 / 6   & 74 / 59 \\
click\_bell         & 79 / 6   & 83 / 86 \\
\midrule
Average & 61.3 / 14.2 & \textbf{65.0 / 42.7} \\
\bottomrule
\end{tabular}
\caption{RoboTwin 2.0 success rate (\%); each cell is \textbf{clean / randomized}. $\pi_{0.5}$ and TFGCA (ours) are run under the same protocol (Aloha-AgileX, 30k steps, 100 evaluations each). The last row is the 6-task average. A full comparison with leaderboard baselines (DP, ACT, DP3, RDT, $\pi_0$) is given in Appendix~D.}
\label{tab:robotwin}
\end{table}

\paragraph{TFGCA improves the base on both clean and randomized, and more so under OOD.} Relative to $\pi_{0.5}$, TFGCA improves the six-task mean by \textbf{$+3.7$} on clean and \textbf{$+28.5$} on randomized (randomization is never trained on, so the randomized column measures zero-shot generalization). Its randomized average of 42.7 is also the highest among all methods in the Appendix~D comparison---about 24 points above the next best---while non-pretrained baselines collapse almost entirely under randomization. The most dramatic are click\_bell and stack\_bowls\_three: $\pi_{0.5}$ nearly collapses under randomization ($\approx$6\%), whereas TFGCA holds \textbf{86\%} and \textbf{59\%}.

\paragraph{The largest randomized gains are on contact/multi-stage tasks---a correlation, not causal proof.} These tasks (click\_bell, stack\_bowls\_three, turn\_switch) are consistent in direction with the Section~\ref{subsec:theory} motivation, but this is only a task-level correlation---handover\_block dips from 13 to 11 on randomized, within the binomial confidence interval at $n{=}100$ and better read as noise. The only two clean regressions (open\_microwave, stack\_bowls\_three) also occur on near-saturated tasks and still show clear randomized gains, so they look like saturation noise. Component net contributions are assessed by the Section~\ref{subsec:libero} ablation.

\subsection{AgiBot A2: Real-Robot Validation}
\label{subsec:agibot}

We further conduct a real-robot evaluation on AgiBot A2, comparing TFGCA against the $\pi_{0.5}$ base on three manipulation tasks (Table~\ref{tab:agibot}). Each task is run 20 times, for 60 trials in total; the phase decomposition, success criteria, and test protocol of each task are given in Appendix~F.

\begin{table}[t]
\centering
\footnotesize
\setlength{\tabcolsep}{2pt}
\begin{tabular}{lccc}
\toprule
Task & TFGCA & $\pi_{0.5}$ & $\Delta$ \\
\midrule
Soap into box & 6/20 (30\%)  & 4/20 (20\%)  & +10 pts \\
Plush toys into basket & 19/20 (95\%) & 17/20 (85\%) & +10 pts \\
Pull a tissue      & 12/20 (60\%) & 9/20 (45\%)  & +15 pts \\
\midrule
\textbf{Overall} & \textbf{37/60 (61.67\%)} & \textbf{30/60 (50.00\%)} & \textbf{+11.67 pts} \\
\bottomrule
\end{tabular}
\caption{AgiBot A2 real-robot success rate.}
\label{tab:agibot}
\end{table}

Across the three real-robot tasks, TFGCA attains an overall success rate of \textbf{61.67\%}, above the $\pi_{0.5}$ base's \textbf{50.00\%}, an improvement of \textbf{11.67 percentage points}. All three tasks improve, with the largest gain on the pull-a-tissue-from-the-box task ($+15$ points). This indicates that TFGCA improves multi-stage manipulation and object interaction on a real robot, beyond the simulation benchmarks.

\section{Discussion and Limitations}
\label{sec:discussion}

\paragraph{Boundary of the evidence.} Identity initialization (Section~\ref{subsec:init}) only guarantees that the module matches the base pointwise \emph{at initialization}---a statement about attachment safety, not a theoretical guarantee that the jointly fine-tuned model beats the base. Effectiveness is instead established empirically: across all four settings TFGCA improves over the same-source base in a consistent direction (Section~\ref{sec:exp}), and the occasional small regressions on near-saturated tasks (Section~\ref{subsec:robotwin}) do not change this overall conclusion.

\paragraph{Limitations.} \emph{The frequency framing is weaker at short horizons}: LIBERO's $T=10$ and single SWT level make the decomposition close to a single trend/detail split, so the frequency narrative is more solid on the $T=50$, two-level, second-order-difference RoboTwin setting.

\paragraph{Future work.} First, \emph{linking TO-DoL to the wedge channel}: a $Q/K$ routing analysis can test whether the wedge really moves attention mass onto the cross-phase token pairs with ``small dot but large wedge,'' and, paired with a dot-only ablation under OOD, could upgrade the mechanism from a motivating hypothesis to evidence. Second, \emph{task-dependent inference frequency}: how often a chunked policy should re-plan appears to interact with a task's reactivity, and the module's frequency view may help predict a good setting. Third, \emph{chunk-boundary diagnostics}: chunked policies jump where consecutive chunks disagree, and a boundary-discontinuity metric (e.g., spectral arc length in kinematics~\citep{smoothness2015}) would make this observable and connect the frequency view to a concrete smoothness quantity.

\section{Conclusion}

Chunked VLA policies under-use the structure of the trajectories they emit, missing both the multi-scale frequency content of motion and the cross-phase, near-orthogonal relationships within a chunk that dot-product attention can retrieve only \emph{indirectly} (it is least sensitive exactly near orthogonality and must re-encode such relationships as similarity). Time--Frequency Geometric Cross-Attention addresses both with a single drop-in module: a per-dimension control-space wavelet tokenization of the action chunk (jointly encoding the multi-dimensional action content at each time--band position), and a cross-attention in which each time token retrieves information from the whole chunk's time--frequency tokens through a learnable blend of similarity and geometric difference. The module uses a zero-initialized residual, so it reproduces the base output exactly at initialization, easing drop-in onto a pretrained policy and joint fine-tuning with it. Empirically, TFGCA raises overall success on RoboTwin and improves more under domain randomization; on LIBERO it yields an average gain on a near-saturated benchmark, and on its OOD variant LIBERO-Plus it improves the Total by $+6.3$ over the same-source base; on the real AgiBot A2 robot the overall success rate rises from \textbf{50.00\%} to \textbf{61.67\%}. These gains are consistent in direction under varying degrees of out-of-distribution conditions.

\newpage
\bibliography{arxiv}

\begin{thebibliography}{33}
\providecommand{\natexlab}[1]{#1}

\bibitem[{Alayrac et~al.(2022)Alayrac, Donahue, Luc, Miech, Barr, Hasson, Lenc, Mensch, Millican, Reynolds et~al.}]{flamingo2022}
Alayrac, J.-B.; Donahue, J.; Luc, P.; Miech, A.; Barr, I.; Hasson, Y.; Lenc, K.; Mensch, A.; Millican, K.; Reynolds, M.; et~al. 2022.
\newblock {Flamingo}: A Visual Language Model for Few-Shot Learning.
\newblock In \emph{Advances in Neural Information Processing Systems (NeurIPS)}.

\bibitem[{Bachlechner et~al.(2021)Bachlechner, Majumder, Mao, Cottrell, and McAuley}]{rezero2021}
Bachlechner, T.; Majumder, B.~P.; Mao, H.~H.; Cottrell, G.~W.; and McAuley, J. 2021.
\newblock {ReZero} is All You Need: Fast Convergence at Large Depth.
\newblock In \emph{Conference on Uncertainty in Artificial Intelligence (UAI)}.

\bibitem[{Balasubramanian et~al.(2015)Balasubramanian, Melendez-Calderon, Roby-Brami, and Burdet}]{smoothness2015}
Balasubramanian, S.; Melendez-Calderon, A.; Roby-Brami, A.; and Burdet, E. 2015.
\newblock On the Analysis of Movement Smoothness.
\newblock \emph{Journal of NeuroEngineering and Rehabilitation}, 12(112).

\bibitem[{Black et~al.(2024)}]{black2024pi0}
Black, K.; et~al. 2024.
\newblock {$\pi_0$}: A Vision-Language-Action Flow Model for General Robot Control.
\newblock \emph{arXiv preprint arXiv:2410.24164}.

\bibitem[{Brandstetter et~al.(2023)Brandstetter, van~den Berg, Welling, and Gupta}]{clifford2023}
Brandstetter, J.; van~den Berg, R.; Welling, M.; and Gupta, J.~K. 2023.
\newblock Clifford Neural Layers for {PDE} Modeling.
\newblock In \emph{International Conference on Learning Representations (ICLR)}.

\bibitem[{Brehmer et~al.(2023)Brehmer, de~Haan, Behrends, and Cohen}]{gatr2023}
Brehmer, J.; de~Haan, P.; Behrends, S.; and Cohen, T. 2023.
\newblock Geometric Algebra Transformer.
\newblock In \emph{Advances in Neural Information Processing Systems (NeurIPS)}.

\bibitem[{Brohan et~al.(2023)}]{rt2_2023}
Brohan, A.; et~al. 2023.
\newblock {RT-2}: Vision-Language-Action Models Transfer Web Knowledge to Robotic Control.
\newblock In \emph{Conference on Robot Learning (CoRL)}.

\bibitem[{Chen et~al.(2025{\natexlab{a}})Chen, Luong, Mukherjee, and Singh}]{simpletm2025}
Chen, H.; Luong, V.; Mukherjee, L.; and Singh, V. 2025{\natexlab{a}}.
\newblock {SimpleTM}: A Simple Baseline for Multivariate Time Series Forecasting.
\newblock In \emph{International Conference on Learning Representations (ICLR)}.

\bibitem[{Chen et~al.(2025{\natexlab{b}})}]{robotwin2025}
Chen, T.; et~al. 2025{\natexlab{b}}.
\newblock {RoboTwin 2.0}: A Scalable Data Generator and Benchmark with Strong Domain Randomization for Robust Bimanual Robotic Manipulation.
\newblock \emph{arXiv preprint arXiv:2506.18088}.

\bibitem[{Chi et~al.(2023)Chi, Xu, Feng, Cousineau, Du, Burchfiel, Tedrake, and Song}]{diffusionpolicy2023}
Chi, C.; Xu, Z.; Feng, S.; Cousineau, E.; Du, Y.; Burchfiel, B.; Tedrake, R.; and Song, S. 2023.
\newblock Diffusion Policy: Visuomotor Policy Learning via Action Diffusion.
\newblock In \emph{Robotics: Science and Systems (RSS)}.
\newblock Extended in International Journal of Robotics Research 44(10--11):1684--1704, 2025.

\bibitem[{Fei et~al.(2025)Fei, Wang, Shi, Dai, Cai, Qian, Ji, He, Zhang, Fei, Fu, Gong, and Qiu}]{liberoplus2025}
Fei, S.; Wang, S.; Shi, J.; Dai, Z.; Cai, J.; Qian, P.; Ji, L.; He, X.; Zhang, S.; Fei, Z.; Fu, J.; Gong, J.; and Qiu, X. 2025.
\newblock {LIBERO-Plus}: In-depth Robustness Analysis of Vision-Language-Action Models.
\newblock \emph{arXiv preprint arXiv:2510.13626}.

\bibitem[{Hu et~al.(2022)Hu, Shen, Wallis, Allen-Zhu, Li, Wang, Wang, and Chen}]{lora2022}
Hu, E.~J.; Shen, Y.; Wallis, P.; Allen-Zhu, Z.; Li, Y.; Wang, S.; Wang, L.; and Chen, W. 2022.
\newblock {LoRA}: Low-Rank Adaptation of Large Language Models.
\newblock In \emph{International Conference on Learning Representations (ICLR)}.

\bibitem[{Kim, Finn, and Liang(2025)}]{openvlaoft2025}
Kim, M.~J.; Finn, C.; and Liang, P. 2025.
\newblock Fine-Tuning Vision-Language-Action Models: Optimizing Speed and Success.
\newblock In \emph{Robotics: Science and Systems (RSS)}.

\bibitem[{Kim et~al.(2024)}]{openvla2024}
Kim, M.~J.; et~al. 2024.
\newblock {OpenVLA}: An Open-Source Vision-Language-Action Model.
\newblock In \emph{Conference on Robot Learning (CoRL)}.

\bibitem[{Lipman et~al.(2023)Lipman, Chen, Ben-Hamu, Nickel, and Le}]{flowmatching2023}
Lipman, Y.; Chen, R. T.~Q.; Ben-Hamu, H.; Nickel, M.; and Le, M. 2023.
\newblock Flow Matching for Generative Modeling.
\newblock In \emph{International Conference on Learning Representations (ICLR)}.

\bibitem[{Liu et~al.(2023)Liu, Zhu, Gao, Feng, Liu, Zhu, and Stone}]{libero2023}
Liu, B.; Zhu, Y.; Gao, C.; Feng, Y.; Liu, Q.; Zhu, Y.; and Stone, P. 2023.
\newblock {LIBERO}: Benchmarking Knowledge Transfer for Lifelong Robot Learning.
\newblock In \emph{Advances in Neural Information Processing Systems (NeurIPS), Datasets and Benchmarks Track}.

\bibitem[{Liu et~al.(2025)}]{rdt2024}
Liu, S.; et~al. 2025.
\newblock {RDT-1B}: A Diffusion Foundation Model for Bimanual Manipulation.
\newblock In \emph{International Conference on Learning Representations (ICLR)}.

\bibitem[{Liu et~al.(2024)Liu, Hu, Zhang, Wu, Wang, Ma, and Long}]{itransformer2024}
Liu, Y.; Hu, T.; Zhang, H.; Wu, H.; Wang, S.; Ma, L.; and Long, M. 2024.
\newblock {iTransformer}: Inverted Transformers Are Effective for Time Series Forecasting.
\newblock In \emph{International Conference on Learning Representations (ICLR)}.

\bibitem[{{NVIDIA}(2025)}]{groot_n16_2026}
{NVIDIA}. 2025.
\newblock {GR00T N1.6}: An Improved Open Foundation Model for Generalist Humanoid Robots.
\newblock \url{https://research.nvidia.com/labs/gear/gr00t-n1_6/}.

\bibitem[{{NVIDIA} et~al.(2025){NVIDIA}, Bjorck, Casta{\~n}eda, Cherniadev, Da, Ding, Fan, Fang, Fox et~al.}]{groot_n1_2025}
{NVIDIA}; Bjorck, J.; Casta{\~n}eda, F.; Cherniadev, N.; Da, X.; Ding, R.; Fan, L.; Fang, Y.; Fox, D.; et~al. 2025.
\newblock {GR00T N1}: An Open Foundation Model for Generalist Humanoid Robots.
\newblock \emph{arXiv preprint arXiv:2503.14734}.

\bibitem[{{Octo Model Team}(2024)}]{octo2024}
{Octo Model Team}. 2024.
\newblock {Octo}: An Open-Source Generalist Robot Policy.
\newblock In \emph{Robotics: Science and Systems (RSS)}.

\bibitem[{Pertsch et~al.(2025)Pertsch, Stachowicz, Ichter, Driess, Nair, Vuong, Mees, Finn, and Levine}]{fast2025}
Pertsch, K.; Stachowicz, K.; Ichter, B.; Driess, D.; Nair, S.; Vuong, Q.; Mees, O.; Finn, C.; and Levine, S. 2025.
\newblock {FAST}: Efficient Action Tokenization for Vision-Language-Action Models.
\newblock \emph{arXiv preprint arXiv:2501.09747}.

\bibitem[{{Physical Intelligence} et~al.(2025){Physical Intelligence}, Black, Brown, Darpinian, Dhabalia, Driess, Esmail, Equi, Finn, Fusai et~al.}]{pi05_2025}
{Physical Intelligence}; Black, K.; Brown, N.; Darpinian, J.; Dhabalia, K.; Driess, D.; Esmail, A.; Equi, M.; Finn, C.; Fusai, N.; et~al. 2025.
\newblock {$\pi_{0.5}$}: A Vision-Language-Action Model with Open-World Generalization.
\newblock \emph{arXiv preprint arXiv:2504.16054}.

\bibitem[{Qu et~al.(2025)Qu, Song, Chen, Yao, Ye, Ding, Wang, Gu, Zhao, Wang et~al.}]{spatialvla2025}
Qu, D.; Song, H.; Chen, Q.; Yao, Y.; Ye, X.; Ding, Y.; Wang, Z.; Gu, J.; Zhao, B.; Wang, D.; et~al. 2025.
\newblock {SpatialVLA}: Exploring Spatial Representations for Visual-Language-Action Model.
\newblock In \emph{Robotics: Science and Systems (RSS)}.

\bibitem[{Wu et~al.(2021)Wu, Xu, Wang, and Long}]{autoformer2021}
Wu, H.; Xu, J.; Wang, J.; and Long, M. 2021.
\newblock {Autoformer}: Decomposition Transformers with Auto-Correlation for Long-Term Series Forecasting.
\newblock In \emph{Advances in Neural Information Processing Systems (NeurIPS)}.

\bibitem[{Yi et~al.(2023)}]{frets2023}
Yi, K.; et~al. 2023.
\newblock Frequency-Domain {MLP}s Are More Effective Learners in Time Series Forecasting.
\newblock In \emph{Advances in Neural Information Processing Systems (NeurIPS)}.

\bibitem[{Yuan et~al.(2026)Yuan, Dong, Liu, and Zhao}]{fastwam2026}
Yuan, T.; Dong, Z.; Liu, Y.; and Zhao, H. 2026.
\newblock {Fast-WAM}: Do World Action Models Need Test-Time Future Imagination?
\newblock \emph{arXiv preprint arXiv:2603.16666}.

\bibitem[{Ze et~al.(2024)Ze, Zhang, Zhang, Hu, Wang, and Xu}]{dp3_2024}
Ze, Y.; Zhang, G.; Zhang, K.; Hu, C.; Wang, M.; and Xu, H. 2024.
\newblock {3D} Diffusion Policy: Generalizable Visuomotor Policy Learning via Simple {3D} Representations.
\newblock In \emph{Robotics: Science and Systems (RSS)}.

\bibitem[{Zhang, Rao, and Agrawala(2023)}]{controlnet2023}
Zhang, L.; Rao, A.; and Agrawala, M. 2023.
\newblock Adding Conditional Control to Text-to-Image Diffusion Models.
\newblock In \emph{IEEE/CVF International Conference on Computer Vision (ICCV)}.

\bibitem[{Zhao et~al.(2025)Zhao, Lu, Kim, Fu, Zhang, Wu, Li, Ma, Han, Finn et~al.}]{cotvla2025}
Zhao, Q.; Lu, Y.; Kim, M.~J.; Fu, Z.; Zhang, Z.; Wu, Y.; Li, Z.; Ma, Q.; Han, S.; Finn, C.; et~al. 2025.
\newblock {CoT-VLA}: Visual Chain-of-Thought Reasoning for Vision-Language-Action Models.
\newblock In \emph{IEEE/CVF Conference on Computer Vision and Pattern Recognition (CVPR)}.

\bibitem[{Zhao et~al.(2023)Zhao, Kumar, Levine, and Finn}]{act2023}
Zhao, T.~Z.; Kumar, V.; Levine, S.; and Finn, C. 2023.
\newblock Learning Fine-Grained Bimanual Manipulation with Low-Cost Hardware.
\newblock In \emph{Robotics: Science and Systems (RSS)}.

\bibitem[{Zheng et~al.(2026)Zheng, Li, Wang, Liu, Kang, Feng, Zheng, Zou, Chen, Zeng et~al.}]{xvla2025}
Zheng, J.; Li, J.; Wang, Z.; Liu, D.; Kang, X.; Feng, Y.; Zheng, Y.; Zou, J.; Chen, Y.; Zeng, J.; et~al. 2026.
\newblock {X-VLA}: Soft-Prompted Transformer as Scalable Cross-Embodiment Vision-Language-Action Model.
\newblock In \emph{International Conference on Learning Representations (ICLR)}.

\bibitem[{Zhou et~al.(2022)Zhou, Ma, Wen, Wang, Sun, and Jin}]{fedformer2022}
Zhou, T.; Ma, Z.; Wen, Q.; Wang, X.; Sun, L.; and Jin, R. 2022.
\newblock {FEDformer}: Frequency Enhanced Decomposed Transformer for Long-Term Series Forecasting.
\newblock In \emph{International Conference on Machine Learning (ICML)}.

\end{thebibliography}

\clearpage
\noindent This document provides the dataset configurations, training and evaluation settings, resource usage, and the Temporal Orthogonal Division-of-Labor (TO-DoL) analysis that are omitted from the main paper for space. Unless otherwise noted, the simulated experiments use LeRobot as the common infrastructure for data organization, policy training, and rollout evaluation; TFGCA and the corresponding $\pi_{0.5}$ base share data processing, training budget, and evaluation conditions. Method definitions, theoretical properties, and the main results are given in the main paper and are not repeated here.

\appendix

\section{Datasets}

\paragraph{LIBERO.} The LIBERO experiments use the LeRobot Hugging Face dataset \texttt{lerobot/libero}, read and organized in the LeRobot data format. Evaluation covers the four standard suites LIBERO-Long, LIBERO-Goal, LIBERO-Object, and LIBERO-Spatial. Each sample contains two $256\times256$ RGB images, an 8-dimensional robot state, and a 7-dimensional action; that action is \emph{not} seven joint angles but a Cartesian end-effector increment (3-D translation, 3-D rotation) plus a 1-D binary gripper command. Images use identity normalization; states and actions use quantile normalization. Training and evaluation episodes follow the LeRobot dataset documentation and the standard LIBERO benchmark conventions.

\paragraph{RoboTwin 2.0.} The RoboTwin 2.0 experiments use the LeRobot Hugging Face dataset \texttt{lerobot/robotwin\_unified}, organized in the LeRobot format. Each sample contains three $480\times640$ RGB images, a 14-dimensional robot state, and a 14-dimensional joint-space action (dual-arm \texttt{qpos}). Policies are trained only on clean demonstrations and evaluated separately in clean and randomized environments, to probe in-distribution performance and generalization to unseen environmental perturbations. Episode setup follows the LeRobot dataset documentation and the RoboTwin benchmark conventions.

\paragraph{AgiBot A2.} The real-robot data were collected in-house and cover three tasks: placing soap into a soap box, placing two plush toys into a basket, and pulling a tissue from a tissue box. Owing to confidentiality constraints, we do not release the raw data or the collection procedure; public documentation of the A2 platform's software/hardware and developer interfaces is available from AgiBot's official developer materials.

\section{Training and Evaluation Configuration}

\paragraph{Common implementation.} All simulated training and evaluation are run in LeRobot, including the policy interface, data loading, feature normalization, action post-processing, and environment rollout. TFGCA is attached between the VLA transformer output and the action head, with base model \texttt{lerobot/pi05\_base}. Apart from TFGCA-specific settings, the augmented model and its $\pi_{0.5}$ base use the same data, training budget, and evaluation pipeline. The common training configuration is given in Table~\ref{tab:train}.

\begin{table}[t]
\centering
\footnotesize
\begin{tabular}{@{}lp{0.52\columnwidth}@{}}
\toprule
Setting & Value \\
\midrule
Base model & \texttt{lerobot/pi05\_base} \\
Optimizer & AdamW \\
Initial learning rate & $2.5\times10^{-5}$ \\
Adam betas & $(0.9, 0.95)$ \\
Adam epsilon & $10^{-8}$ \\
Weight decay & 0.01 \\
Gradient clipping & 1.0 \\
Warmup & 1{,}000 steps \\
LR schedule & cosine to $2.5\times10^{-6}$ at 30k steps \\
Numerical precision & bfloat16 \\
Memory optimization & gradient checkpointing \\
Random seed & 1000 \\
FM time sampling & $\mathrm{Beta}(1.5,1.0)$, scale 0.999, shift 0.001 \\
Inference denoising steps & 10 \\
Checkpoint interval & 1{,}000 steps \\
\bottomrule
\end{tabular}
\caption{Common training configuration shared by TFGCA and the $\pi_{0.5}$ base.}
\label{tab:train}
\end{table}

\paragraph{Per-benchmark configuration.} The training and evaluation configuration for the three simulation benchmarks (LIBERO, LIBERO-Plus, RoboTwin 2.0) is given in Table~\ref{tab:perbench}; other common hyperparameters are in Table~\ref{tab:train}. On each benchmark TFGCA and the corresponding reproduced $\pi_{0.5}$ base share consistent data, training budget, and evaluation, differing only in the TFGCA-specific SWT/geometry configuration. The action-space alignment loss uses the ground-truth action as its target with a stop-gradient on the target branch. The difference order counts the per-dimension sequence out of \texttt{action\_proj} as order 0: LIBERO and LIBERO-Plus feed the SWT directly, while RoboTwin 2.0 applies two finite differences before the SWT. Success rates are computed as the fraction of rollouts that strictly complete the task. \textbf{LIBERO-Plus is not trained separately}: the LIBERO-trained models are evaluated on its out-of-distribution perturbations to measure zero-shot generalization; the construction, parameter ranges, and success criteria of the seven perturbation families (camera viewpoint, robot initialization, language, lighting, background, sensor noise, object layout) follow the official LIBERO-Plus setting~\citep{liberoplus2025}.

\begin{table}[t]
\centering
\footnotesize
\setlength{\tabcolsep}{3pt}
\begin{tabular}{@{}p{0.26\columnwidth}ccc@{}}
\toprule
Setting & LIBERO & LIBERO-Plus & RoboTwin \\
\midrule
Training data & \texttt{libero} & reuse & \texttt{robotwin} \\
Training steps & 30k & --- & 30k/task \\
Batch (per GPU / global) & 32 / 64 & --- & 32 / 64 \\
Chunk / exec.\ length & 10 / 10 & 10 / 10 & 50 / 50 \\
SWT levels & 1 & 1 & 2 \\
Diff.\ order before SWT & 0 & 0 & 2 \\
Mother wavelet & \texttt{db2} & \texttt{db2} & \texttt{db2} \\
Attention heads & 8 & 8 & 8 \\
Alignment-loss weight & 0.01 & 0.01 & 0.1 \\
Rollouts / condition & conv. & official & 100 each \\
\bottomrule
\end{tabular}
\caption{Per-benchmark training and evaluation configuration. Filters are learnable in all cases (initialized to \texttt{db2}). ``reuse'' means LIBERO-Plus reuses the LIBERO-trained model without separate training; ``30k/task'' means RoboTwin is trained 30k steps per task; ``conv.'' follows the LIBERO benchmark convention for rollouts per condition.}
\label{tab:perbench}
\end{table}

\paragraph{How the TFGCA-specific settings are chosen.} We do not run a large hyperparameter sweep; the few TFGCA-specific settings in Table~\ref{tab:perbench} are fixed from the data and the base configuration rather than tuned per benchmark. The chunk and execution length follow the LeRobot base configuration for each benchmark (10 for LIBERO/LIBERO-Plus, 50 for RoboTwin). The number of SWT levels is tied to the chunk length---one level for the short LIBERO chunk, two for the longer RoboTwin chunk, so that the coarsest band still covers a meaningful fraction of the chunk. The difference order before the SWT is set from the energy distribution of the action sequence: LIBERO/LIBERO-Plus end-effector deltas are already near-stationary and are fed to the SWT directly (order 0), whereas RoboTwin joint-space \texttt{qpos} carries a strong low-frequency drift, so we apply two finite differences (order 2) to move its energy into the detail bands. The alignment-loss weight is $0.01$ in the multi-task simulation settings (LIBERO, LIBERO-Plus) and $0.1$ in the single-task RoboTwin setting, where a stronger alignment target is affordable without competing across tasks. The remaining values (optimizer, learning rate, schedule, batch size, denoising steps, seed) are shared with the $\pi_{0.5}$ base and are not re-tuned for TFGCA.

\section{LIBERO: Extended Comparison}

Table~1 of the main paper reports TFGCA, the reproduced $\pi_{0.5}$ base, and the ablations under a single protocol. Table~\ref{tab:liberofull} places these alongside external methods for context. As these values come from the original papers or recent public comparisons under possibly different evaluation protocols, they are provided \emph{for reference only} and should not be read as a strict ranking against TFGCA.

\begin{table*}[t]
\centering
\small
\setlength{\tabcolsep}{4pt}
\begin{tabular}{lccccc}
\toprule
Method & Spatial & Object & Goal & Long & Avg \\
\midrule
Diffusion Policy~\citep{diffusionpolicy2023} & 78.5 & 87.5 & 73.5 & 64.8 & 76.1 \\
OpenVLA~\citep{openvla2024} & 84.7 & 88.4 & 79.2 & 53.7 & 76.5 \\
SpatialVLA~\citep{spatialvla2025} & 88.2 & 89.9 & 78.6 & 55.5 & 78.1 \\
CoT-VLA~\citep{cotvla2025} & 87.5 & 91.6 & 87.6 & 69.0 & 83.9 \\
$\pi_0$-FAST~\citep{fast2025} & 96.4 & 96.8 & 88.6 & 60.2 & 85.5 \\
GR00T-N1~\citep{groot_n1_2025} & 94.4 & 97.6 & 93.0 & 90.6 & 93.9 \\
$\pi_0$~\citep{black2024pi0} & 98.0 & 96.8 & 94.4 & 88.4 & 94.4 \\
$\pi_{0.5}$~\citep{pi05_2025} & 98.8 & 98.2 & 98.0 & 92.4 & 96.9 \\
GR00T-N1.6~\citep{groot_n16_2026} & 97.7 & 98.5 & 97.5 & 94.4 & 97.0 \\
OpenVLA-OFT~\citep{openvlaoft2025} & 97.6 & 98.4 & 97.9 & 94.5 & 97.1 \\
Fast-WAM~\citep{fastwam2026} & 98.2 & 100.0 & 97.0 & 95.2 & 97.6 \\
X-VLA~\citep{xvla2025} & 98.2 & 98.6 & 97.8 & 97.6 & 98.1 \\
\midrule
$\pi_{0.5}$ (our repro.) & 95.2 & 99.6 & 97.2 & 94.6 & 96.7 \\
\textbf{TFGCA (full)} & \textbf{98.5} & \textbf{99.4} & \textbf{97.7} & \textbf{97.0} & \textbf{98.2} \\
TFGCA $-$ SWT & 96.8 & 99.1 & 97.1 & 97.8 & 97.7 \\
TFGCA $-$ geometry & 97.2 & 99.7 & 97.6 & 96.5 & 97.8 \\
TFGCA $-$ alignment & 94.9 & 99.5 & 98.5 & 95.5 & 97.1 \\
\bottomrule
\end{tabular}
\caption{LIBERO success rate (SR\%), TFGCA and $\pi_{0.5}$ (our reproduction and ablations, same protocol) alongside external methods (upper block, reference only---protocols may differ).}
\label{tab:liberofull}
\end{table*}

\section{RoboTwin 2.0: Full Comparison}

Table~3 of the main paper compares TFGCA against the $\pi_{0.5}$ base under a single protocol. Table~\ref{tab:robotwinfull} places these alongside the RoboTwin 2.0 official leaderboard baselines. $\pi_{0.5}$ and TFGCA (ours) are run under the same protocol (Aloha-AgileX, 30k steps, 100 evaluations each) and are directly comparable; DP, ACT, DP3, RDT, and $\pi_0$ are the leaderboard-reported values. DP3 uses point-cloud input, a different modality from the other RGB methods.

\begin{table*}[t]
\centering
\small
\setlength{\tabcolsep}{5pt}
\begin{tabular}{lccccccc}
\toprule
Task & DP & ACT & DP3 & RDT & $\pi_0$ & $\pi_{0.5}$ & \textbf{TFGCA (ours)} \\
\midrule
open\_microwave    & 5 / 0   & 86 / 0  & 61 / 22 & 37 / 20  & 80 / 50  & 93 / 27  & 88 / 36 \\
stamp\_seal        & 2 / 0   & 2 / 0   & 18 / 0  & 1 / 0    & 3 / 4    & 15 / 5   & 19 / 3 \\
handover\_block    & 10 / 0  & 42 / 0  & 70 / 0  & 45 / 14  & 45 / 8   & 57 / 13  & 72 / 11 \\
turn\_switch       & 36 / 1  & 5 / 2   & 46 / 8  & 35 / 15  & 27 / 23  & 41 / 28  & 54 / 61 \\
stack\_bowls\_three & 63 / 0  & 48 / 0  & 57 / 5  & 51 / 17  & 66 / 24  & 83 / 6   & 74 / 59 \\
click\_bell        & 54 / 0  & 58 / 3  & 90 / 0  & 80 / 9   & 44 / 3   & 79 / 6   & 83 / 86 \\
\midrule
Average & 28.3 / 0.2 & 40.2 / 0.8 & 57.0 / 5.8 & 41.5 / 12.5 & 44.2 / 18.7 & 61.3 / 14.2 & \textbf{65.0 / 42.7} \\
\bottomrule
\end{tabular}
\caption{RoboTwin 2.0 success rate (\%); each cell is \textbf{clean / randomized} (the leaderboard's Easy / Hard). Values for DP~\citep{diffusionpolicy2023}, ACT~\citep{act2023}, DP3~\citep{dp3_2024}, RDT~\citep{rdt2024}, and $\pi_0$~\citep{black2024pi0} are from the RoboTwin 2.0 official leaderboard; $\pi_{0.5}$ and TFGCA (ours) are run under the same protocol and directly comparable. The last row is the 6-task average.}
\label{tab:robotwinfull}
\end{table*}

\section{LIBERO-Plus: Per-Suite Full Results}

Table~2 of the main paper reports only the average of the seven LIBERO-Plus perturbations across the four suites. Table~\ref{tab:liberoplusfull} gives TFGCA's full per-perturbation success rate (\%) on each of the four suites (Spatial, Object, Goal, Long); each cell is TFGCA, with the $\pi_{0.5}$ base and the delta relative to it in parentheses. The final Avg row is the four-suite average and corresponds to the two LIBERO-Plus rows of main-paper Table~2. Both models are self-tested under the same protocol.

\begin{table*}[t]
\centering
\scriptsize
\setlength{\tabcolsep}{2.5pt}
\begin{tabular}{@{}lcccccccc@{}}
\toprule
Suite & Camera & Robot & Language & Light & Background & Noise & Layout & Total \\
\midrule
Spatial & 68.3 (59.3, +9.0) & 88.6 (80.9, +7.7) & 87.7 (81.8, +5.9) & 99.0 (97.3, +1.7) & 99.2 (94.2, +5.0) & 65.8 (46.7, +19.1) & 97.4 (95.1, +2.3) & 85.8 (78.3, +7.4) \\
Object & 41.7 (33.6, +8.1) & 70.8 (58.5, +12.3) & 94.1 (84.5, +9.6) & 96.6 (100.0, $-$3.4) & 96.4 (97.6, $-$1.2) & 49.3 (40.8, +8.5) & 82.6 (82.9, $-$0.2) & 73.3 (67.9, +5.4) \\
Goal & 60.8 (45.1, +15.7) & 67.0 (58.4, +8.5) & 53.2 (58.3, $-$5.1) & 97.5 (95.7, +1.8) & 75.1 (77.2, $-$2.1) & 48.8 (39.3, +9.5) & 64.2 (60.7, +3.5) & 64.9 (59.9, +4.9) \\
Long & 35.1 (30.3, +4.8) & 67.2 (53.9, +13.3) & 92.2 (84.1, +8.1) & 92.7 (92.3, +0.4) & 80.6 (84.8, $-$4.2) & 45.2 (26.7, +18.5) & 82.4 (79.8, +2.6) & 67.9 (60.7, +7.2) \\
\midrule
\textbf{Avg} & \textbf{51.5 (42.1, +9.4)} & \textbf{73.4 (62.9, +10.5)} & \textbf{81.8 (77.2, +4.6)} & \textbf{96.4 (96.3, +0.1)} & \textbf{87.8 (88.4, $-$0.6)} & \textbf{52.3 (38.4, +13.9)} & \textbf{81.7 (79.6, +2.1)} & \textbf{73.0 (66.7, +6.3)} \\
\bottomrule
\end{tabular}
\caption{LIBERO-Plus per-suite, per-perturbation success rate (\%). Each cell is TFGCA ($\pi_{0.5}$ base, $\Delta$); the Avg row is the four-suite average and corresponds to the two LIBERO-Plus rows of Table~2 in the main paper.}
\label{tab:liberoplusfull}
\end{table*}

By suite, TFGCA's gains remain concentrated on the perturbations the base handles worst---Noise and Camera (e.g., Spatial-Noise $+19.1$, Long-Noise $+18.5$, Goal-Camera $+15.7$)---while on the near-saturated Light/Background it is roughly flat or slightly regresses, within the evaluation sampling variance. This is consistent with the average-level conclusion of Section~5.2 in the main paper.

\section{Real-Robot Task Decomposition (AgiBot A2)}

\paragraph{AgiBot A2 protocol.} TFGCA and the $\pi_{0.5}$ base use the same training data, number of steps, task initialization, and testing conditions. Each model is tested 20 times per task. Evaluation is strict: a single continuous execution counts as a success only if the goal state is fully reached; human intervention, mid-task retries, dropped objects, or the failure of any required sub-step count as failures.

Figures~\ref{fig:soap}--\ref{fig:doll} show a stage-wise keyframe decomposition of the three AgiBot A2 real-robot tasks. Keyframes are taken from a representative rollout and correspond one-to-one with Table~4 of the main paper.

\begin{figure}[t]
\centering
\includegraphics[width=\linewidth]{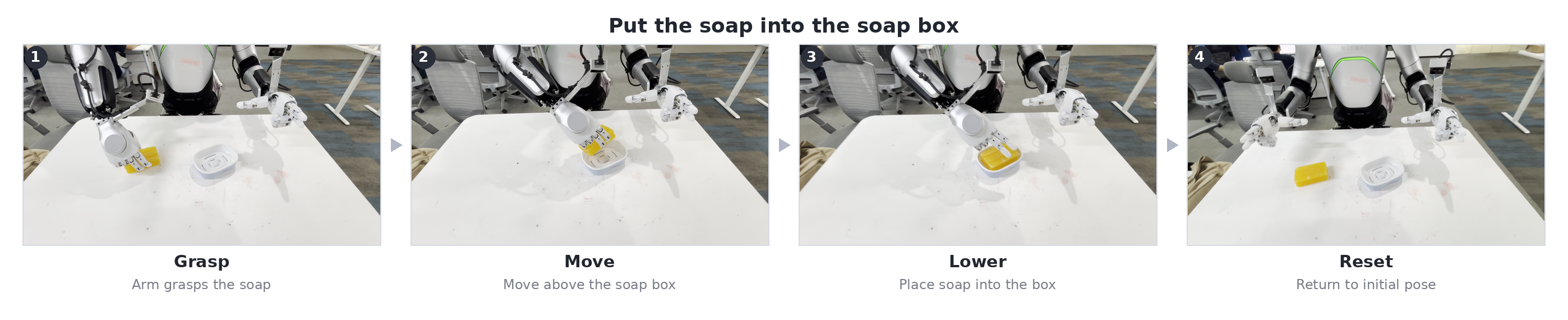}
\caption{Put the soap into the soap box. Four-stage keyframes: grasp the soap $\to$ move above the box $\to$ release into the box $\to$ reset.}
\label{fig:soap}
\end{figure}

\begin{figure}[t]
\centering
\includegraphics[width=\linewidth]{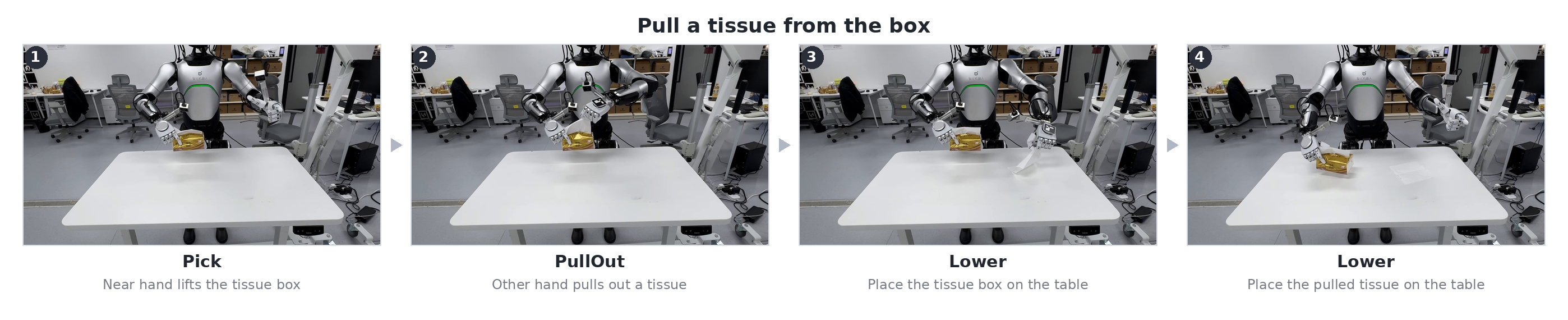}
\caption{Pull a tissue from the box. Four-stage keyframes: the near hand picks up the tissue box $\to$ the other hand pulls a tissue $\to$ put down the box $\to$ put down the tissue.}
\label{fig:tissue}
\end{figure}

\begin{figure}[t]
\centering
\includegraphics[width=\linewidth]{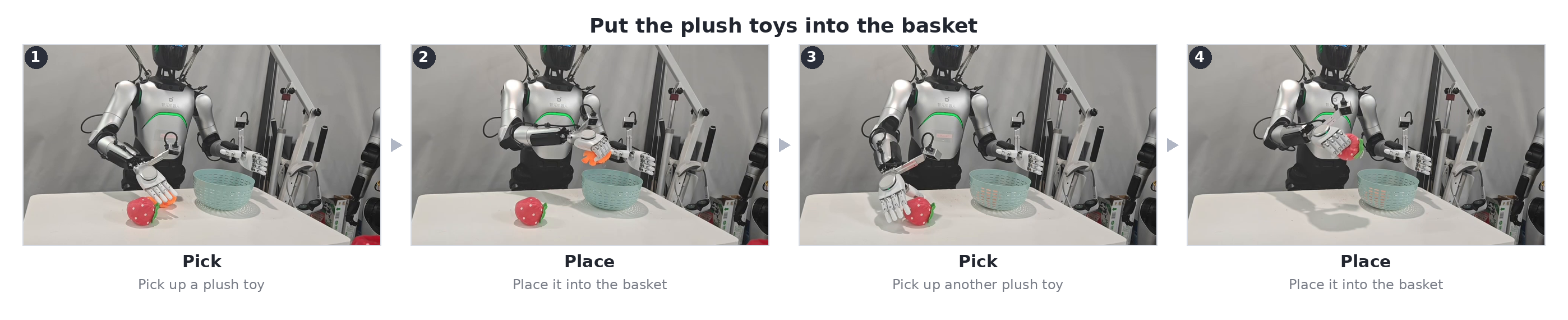}
\caption{Put the two plush toys into the basket. Keyframes: pick up a toy $\to$ place it in the basket, then repeat for the second toy.}
\label{fig:doll}
\end{figure}

\section{Resource Usage}

\subsection{Training Hardware and Wall-Clock}

\begin{table}[t]
\centering
\small
\setlength{\tabcolsep}{5pt}
\begin{tabular}{llr}
\toprule
Experiment & GPU & Wall-clock \\
\midrule
LIBERO & $2\times$ NVIDIA A800 80GB & $\sim$26 hours \\
RoboTwin 2.0 & $2\times$ NVIDIA A800 80GB & $\sim$28 hours \\
\bottomrule
\end{tabular}
\caption{Resource usage per training job (not the cumulative cost of a full benchmark).}
\label{tab:resource}
\end{table}

The resource records in Table~\ref{tab:resource} correspond to each training job rather than the cumulative cost of a whole benchmark. Beyond the measured training hardware and approximate wall-clock, we do not report peak memory, training throughput, end-to-end inference latency, or other quantities that were not independently measured, so as to avoid inferring real running costs from configuration files or theoretical values. Exact software versions were not recorded in a single verifiable log and are therefore not stated speculatively; the AgiBot A2 platform's internal resource configuration is also withheld for confidentiality reasons.

\subsection{Added Parameter Count}

The parameters TFGCA adds over the $\pi_{0.5}$ base are fully determined by the module structure and can be computed exactly from the architecture (no independent measurement needed). Taking model width $d=1024$ and LIBERO's action dimension $D=7$ (6-D Cartesian end-effector increment + 1-D binary gripper), the per-module counts are given in Table~\ref{tab:params}.

\begin{table}[t]
\centering
\small
\setlength{\tabcolsep}{6pt}
\begin{tabular}{lr}
\toprule
Module & Parameters \\
\midrule
\texttt{CrossGeometricAttention} & 4{,}198{,}403 (4.20M) \\
\texttt{PerDimensionSWTEncoder}   & 28{,}728 (0.029M) \\
\texttt{action\_proj}             & 7{,}168 (0.007M) \\
\midrule
\textbf{Total} & \textbf{4{,}234{,}299 (4.23M)} \\
\bottomrule
\end{tabular}
\caption{Added parameters of TFGCA at $d=1024$, $D=7$.}
\label{tab:params}
\end{table}

The added parameters are dominated by the geometric cross-attention's Q/K/V/output projections (four $1024\times1024$ matrices), which do not exist in the $\pi_{0.5}$ base and are the genuinely new part of TFGCA; the per-dimension SWT encoder and the action-space projection (\texttt{PerDimensionSWTEncoder} and \texttt{action\_proj}) are almost negligible. The total of about \textbf{4.23M} is roughly \textbf{0.10\%} of the $\sim$4B-parameter $\pi_{0.5}$ backbone, consistent with the main paper's positioning of TFGCA as a lightweight drop-in module. This count is an architecture-determined exact value and does not change with training data or steps; when the action dimension $D$ varies with the embodiment (e.g., $D=14$ for RoboTwin and AgiBot), only \texttt{action\_proj} scales linearly in $d\times D$, with negligible effect on the total.

\section{Temporal Orthogonal Division-of-Labor (TO-DoL): Definition and Data Analysis}

Sections~1 and~4.5 of the main paper use the claim that ``the near-orthogonal structure of division-of-labor coordination lives along the time axis'' as data support for the design motivation. This section gives the full definition, metrics, and robustness analysis. The analysis covers RoboTwin 2.0's \texttt{handover\_block} and \texttt{stack\_bowls\_three} tasks, 50 demonstrations each, over the aloha-agilex and arx-x5 bimanual embodiments. It is based on existing demonstration trajectories and involves no new training or evaluation.

\subsection{Motivation: Two Notions of Orthogonality}

A natural but untested hypothesis is a \emph{same-instant} version of orthogonality: at some frame $t$, two division-of-labor sub-actions A and B are simultaneously active with joint directions $v_A(t)\perp v_B(t)$. The data nearly falsify this hypothesis (Table~\ref{tab:sameinstant}): same-instant orthogonal coupling is essentially absent. The same data, however, exhibit very strong positive structure \emph{along the time axis}---this is TO-DoL.

\begin{table}[t]
\centering
\footnotesize
\setlength{\tabcolsep}{4pt}
\begin{tabular}{@{}p{0.40\columnwidth}p{0.30\columnwidth}p{0.20\columnwidth}@{}}
\toprule
Test & Observation & Verdict \\
\midrule
$\ge 3$ subspaces co-active within a 10-frame chunk & \textbf{0.0\%} (all four datasets) & almost never ``many things at once'' \\
Cross-arm co-activation rate & 14\% / 13\% & \emph{below} random control (27--30\%) \\
Left/right direction corr.\ in active frames (handover) & $-0.24$ to $-0.37$ & one arm moves, the other pauses \\
\bottomrule
\end{tabular}
\caption{Same-instant orthogonal coupling is near-zero.}
\label{tab:sameinstant}
\end{table}

\subsection{Definition}

Let each trajectory $e$ have per-frame joint velocity $v(t)\in\mathbb{R}^{14}$ and active-frame direction $u(t)=v(t)/\lVert v(t)\rVert$. Segment the trajectory by its \emph{dominant active subspace} into ordered phases $P_1\prec P_2\prec\dots\prec P_K$, where phase $k$'s principal direction $\bar u_k$ is the principal component of $u(t)$ within that phase. The trajectory has \textbf{Temporal Orthogonal Division-of-Labor (TO-DoL)} if the following three conditions hold simultaneously:
\begin{itemize}
\item \textbf{(C1) Inter-phase orthogonality}: $\max_{i\neq j}\lvert\langle\bar u_i,\bar u_j\rangle\rvert\le\varepsilon$, i.e., different temporal phases occupy near-orthogonal joint directions.
\item \textbf{(C2) Non-trivial temporal order}: the phase order $(P_1\prec\dots\prec P_K)$ and relative onset times are stable across trials and \emph{significantly better than a null control}.
\item \textbf{(C3) Within-phase coherence}: within a single phase, the direction $u(t)$ is stable (low variance), so the phase genuinely ``occupies'' a subspace rather than wandering.
\end{itemize}

\paragraph{C2 is the crux.} In high dimensions any two sparse directions with disjoint support are already near-orthogonal, so C1 is nearly geometrically trivial and is not evidence on its own; what genuinely needs data support is C2---the ordering must beat controls that ``shuffle the phase order'' and ``randomly pair across trials''.

\subsection{Condition-by-Condition Verification}

\paragraph{C1 (inter-phase orthogonality) --- holds, but is a trivial baseline.} The 14-D direction cosine between the left-dominant and right-dominant segments: aloha-agilex $\lvert\cos\rvert=0.0013$ (p90 $\approx 0.0017$), arx-x5 $=0.0014$ (p90 $\approx 0.002$). Disjoint coordinate blocks are almost exactly orthogonal by geometry, so the weight of evidence rests on C2.

\paragraph{C2 (non-trivial temporal order) --- holds strongly (Table~\ref{tab:c2}).}

\begin{table}[t]
\centering
\footnotesize
\setlength{\tabcolsep}{4pt}
\begin{tabular}{@{}p{0.46\columnwidth}p{0.26\columnwidth}p{0.16\columnwidth}@{}}
\toprule
Metric & Value & Control \\
\midrule
Left$\to$right phase-order rate (handover) & 100\% of episodes & --- \\
Gripper-event order (L-grasp $<$ R-grasp $<$ L-release) & 100\% (both embodiments) & --- \\
Grasp$\to$release interval & 22 frames, std $=0$ & --- \\
Envelope-preserving, order-shuffled clean control & real $3.5$--$5\times$ tighter than control & --- \\
\bottomrule
\end{tabular}
\caption{C2 evidence: the temporal order is stable and beats the clean control.}
\label{tab:c2}
\end{table}

Two honest corrections must be stated. (1) \emph{The effect size is $3.5$--$5\times$, not higher.} A ``random-phase'' control destroys the entire motion envelope (not just the ordering), inflating the effect size to $9$--$37\times$; using an \emph{envelope-preserving, order-only-shuffled} clean control, the real order structure is $3.5$--$5\times$ tighter, and both the main paper and this section quote the latter. (2) \emph{The coupling sits at the ``population-template'' level, not per-trial feedback.} Randomly \emph{pairing} across episodes (a left-arm trajectory with any right-arm trajectory) almost perfectly reproduces the ordering coupling (ratio $\approx 1.0$). This means each arm follows its own fixed ``temporal script'' and the two scripts happen to be staggered---a shared sequential template rather than one arm closed-loop-adjusting to the other in real time. C2 still holds (the order is significant, stable, and beats the clean control), but is more accurately described as a population-level, script-level temporal coupling.

\paragraph{C3 (within-phase coherence) --- essentially holds.} Each arm's direction is stable within its dominant phase; the handover left arm is bimodal (grasp at $t\approx0$, release at $t\approx0.8$), but both peaks lie in the left-arm subspace and do not break C3---they merely split the left-arm phase into two segments.

\subsection{Robustness}

\begin{itemize}
\item \textbf{Across embodiments}: aloha-agilex and arx-x5 give same-order numbers and consistent conclusions (C1 $\lvert\cos\rvert$ 0.0013 vs 0.0014; C2 order 100\% of episodes; deterministic gripper timing lock).
\item \textbf{Across tasks}: handover (sequential hand-off) \emph{strongly satisfies C2} (clear ordering); stack (left--right symmetric, no strict order, 48--50\%) \emph{weakens C2} but still satisfies C1/C3. The ``order'' strength of C2 is thus task-dependent---the main paper states C2 as ``a stable temporal-phase separation exists'', not universally ``strict order''.
\item \textbf{Across datasets}: LIBERO (end-effector delta, single arm) has no left/right-arm structure and TO-DoL does not apply; but it corroborates the triviality of ``same-instant orthogonality'', which does not conflict with this definition.
\end{itemize}

\section{Identity Initialization: Formal Properties}

Section~4.4 of the main paper constructs the identity initialization from a zero-initialized residual; here we give its formal statement. Let $f_\theta$ be the base policy and $f_{\theta,\phi}$ its TFGCA-augmented version with module parameters $\phi$ and output projection $W_O=0$. Then $\tilde{H}=H+W_O(AV)=H$, giving two complementary properties.
\begin{itemize}
\item \textbf{(a) Safe (non-decreasing capacity).} $f_{\theta,\phi}=f_\theta$ pointwise, so the base policy lies in the augmented hypothesis class and $\min_{\phi}\mathcal{L}(f_{\theta,\phi})\le\mathcal{L}(f_\theta)$: attaching TFGCA cannot worsen the best attainable fit, and it changes no output at initialization.
\item \textbf{(b) Trainable (escaping the identity).} The initial gradient of the loss with respect to $W_O$ is $\nabla_{W_O}\mathcal{L}=\delta\,(AV)^\top$, where $\delta$ is the upstream gradient at the module output. As long as the value tokens satisfy $V\neq 0$ (and $\delta\neq0$), this gradient is generically nonzero, so $W_O$ leaves zero and the module begins to act; if $V=0$ then $AV=0$ forces $\nabla_{W_O}\mathcal{L}=0$, and the gradient to the upstream wavelet parameters also vanishes through $W_O=0$, trapping the module at the identity. This is why we initialize the sub-band embeddings and the action-space projection to be nonzero and zero only $W_O$.
\end{itemize}

\noindent\emph{Proof.} With $W_O=0$ we have $\tilde H=H$ and the head and backbone are unchanged; (a) follows by evaluating the augmented class at $W_O=0$, and (b) follows by differentiating $\tilde H=H+W_O(AV)$ with respect to $W_O$. $\square$

\paragraph{Safety and effectiveness are two different things.} These properties guarantee function preservation \emph{at the moment of initialization}---the module reproduces the base output pointwise on attachment and does not decrease capacity, a statement about attachment risk. They \emph{neither} guarantee that the jointly fine-tuned behavior is no worse than the base (there are indeed a few per-task regressions on clean, Section~5.3 of the main paper) \emph{nor} constitute evidence that the module is effective; effectiveness can only be supported by post-training empirical results (Section~5 of the main paper). The two should not be conflated.

\end{document}